\documentclass[11pt]{article}

\usepackage{fullpage}
\usepackage[varqu,varl,var0,scaled=0.97]{inconsolata}
\usepackage{amsmath,amssymb,amsfonts,amsthm,mathrsfs,mathtools}
\usepackage{hyperref}
\usepackage{url}
\usepackage{subscript}
\usepackage{graphicx}
\usepackage{xspace}

\newtheorem{Th}{Theorem}
\newtheorem{Cor}[Th]{Corollary}

\newtheorem{Lemma}[Th]{Lemma}
\newtheorem{Ex}[Th]{Example}

\theoremstyle{definition}
\newtheorem{Def}{Definition}[section]

\newcommand{\R}{\mathbb{R}}
\newcommand{\N}{\mathbb{N}}

\newcommand{\pp}{\mathcal{P}}
\newcommand{\Y}{\mathcal{Y}}
\newcommand{\pr}{\mathbb{P}}
\newcommand{\W}{\mathcal{W}}

\newcommand{\WPC}{WPC\xspace}
\newcommand{\PC}{PC\xspace}

\renewcommand{\epsilon}{\varepsilon}

\newcommand{\ovl}[1]{\mkern 1.5mu\overline{\mkern-1.5mu#1\mkern-1.5mu}\mkern 1.5mu}

\allowdisplaybreaks

\begin{document}

\title{Universal Consistency of Wasserstein $k$-NN classifiers: Negative and Positive Results}

\author{
    Donlapark Ponnoprat
\hspace*{0.5cm} \\
 Chiang Mai University, Chiang Mai, Thailand.
}

\date{}
\maketitle

\begin{abstract}
    {The Wasserstein distance provides a notion of dissimilarities between probability measures, which has recent applications in learning of structured data with varying size such as images and text documents. In this work, we study the $k$-nearest neighbor classifier ($k$-NN) of probability measures under the Wasserstein distance. We show that the $k$-NN classifier is not universally consistent on the space of measures supported in $(0,1)$. As any Euclidean ball contains a copy of $(0,1)$, one should not expect to obtain universal consistency without some restriction on the base metric space, or the Wasserstein space itself. To this end, via the notion of $\sigma$-finite metric dimension, we show that the $k$-NN classifier is universally consistent on spaces of measures supported in a $\sigma$-uniformly discrete set. In addition, by studying the geodesic structures of the Wasserstein spaces for $p=1$ and $p=2$, we show that the $k$-NN classifier is universally consistent on the space of measures supported on a finite set, the space of Gaussian measures, and the space of measures with densities expressed as finite wavelet series.}
\\
2000 Math Subject Classification: 62H30,  54F45
\end{abstract}

\section{Introduction}
Given a metric space $(X,d)$, the space of probability measures $\pp(X)$ over $X$ and $p\in [1,\infty)$, the $p$-Wasserstein distance on $\pp(X)$ is given by
\begin{equation}\label{eq:wdist}
    W_p(\mu,\nu) = \inf_{\pi\in \Pi(\mu,\nu)}\Big( \int_{X\times X} d(x,y)^p d\pi(x,y)\Big)^{1/p},
\end{equation}
where $\Pi$ is the set of probability measures on $X\times X$ with marginals $\mu$ and $\nu$. It can be shown that $W_p$ is indeed a distance (see~\cite{MAL-073,Santambrogio2015,Panaretos2020} or~\cite{Villani2003} for instance). 

The $p$-Wasserstein distance is connected with the theory of optimal transportation, which have many applications in various fields, such as statistics, machine learning, partial differential equations and economics. The metric itself has been used to measure dissimilarities in high-dimensional data, with most of the focus being on $p=1$ and $2$. For example, text documents can be treated as probability measures over the space of words, and the distances between words are computed from word embedding techniques such as word2vec~\cite{NIPS2013_5021} and GloVe~\cite{pennington2014}. The 1-Wasserstein distance in this setting is called the \textit{Word Mover Distance}~\cite{kusnerb15}. In computer vision, we can use the 1-Wasserstein distance to compute distances between images using the color histograms as probability measures. This so-called \textit{Earth Mover's Distance} has applications in image retrieval~\cite{Rubner}. The case $p=2$ has been used in many imaging tasks due to its intrinsic connection to the Euclidean distance~\cite{Schmitzer2014,Maas2015,Wang2012}; see~\cite{Kolouri2017} for a recent survey of applications.

In this study, we consider the binary classification problem in $(\pp(X),W_p)$. Let $\Y=\{0,1\}$ and $\pr$ be a probability distribution over $\pp(X)\times\Y$, from which instances $(\mu,Y)$ are drawn from. Our goal is to find a \textit{classifier} $g: \pp(X)\to \Y$ that minimizes the \textit{risk function} $R(g)=\pr(g(\mu)\not= Y).$ If we know $\pr$, then it is easy to find the best classifier: let $\eta$ denote the conditional probability $\eta(\mu) = \pr(Y=1|\mu)$, then the \textit{Bayes classifier} $g^*(\mu)=\mathbf{1}_{\eta(\mu)\geq 1/2}$ gives the minimum possible risk, called the \textit{Bayes risk}~\cite{Devroye1996}.
\[
R^* = \pr (g^*(\mu) \not= Y)=\mathbb{E}_{\mu}[\min \{\eta(\mu),1-\eta(\mu)\}].
\]
However, $\pr$ is most likely unknown, so we have to make a classifier based on a finite random sample $D_n=\{(\mu_1,Y_1),\ldots, (\mu_n,Y_n)\}$ drawn independently from $\pr$. The supervised learning approach starts from the \textit{learning rule} $h_n$ : 
\[h_n:(\pp(X)\times\Y)^n\times \pp(X)\to \Y.\]
Then $g_n=h_n(D_n)$ is the classifier that we would like to employ. The performance of $g_n$ is measured by the \textit{error probability}:
\[
R_n = \pr(g_n(\mu)\not= Y|D_n),
\] 
which is a random variable as a function of $D_n$. Obviously, $R_n$ is greater than $R^*$; one of basic questions about the classifier concerns the convergence of the error probability to the Bayes risk as $n\to\infty$. Since $\pr$ is unknown, it is also desirable that the convergence holds \emph{universally}, independent of $\pr$ 
.\begin{Def}[Universal Consistency]
    A classifier $g_n$ is
    \begin{itemize}
        \item universally weakly consistent if $\lim_{n\to\infty} \mathbb{E}[R_n]=R^*$ 
        \item universally strongly consistent if $\lim_{n\to\infty} R_n=R^*$ almost surely 
    \end{itemize}
    for all distribution $\pr$. 
\end{Def}

One of the most well-known classifier is the $k$-nearest neighbor ($k$-NN), which can be equipped with the Wasserstein distance for measure classification. This model can be used to classify documents and image data, which have been preprocessed into probability measures using one of the methods as described in~\cite{kusnerb15} or~\cite{Rubner}. The goal of this work is to analyze and establish the universal consistency of the $k$-NN classifier on a subspace of measures.

Let us take a look at the consistency of $k$-NN in the Euclidean setting. When $k$ is fixed, the limit of $\mathbb{E}[R_n]$ is generally larger than the Bayes risk~\cite{Cover1967,Gyorfi1978}. Thus the consistency of nearest neighbor classified are usually considered when the number of nearest neighbors $k_n$ grows with $n$; we shall call this a $k_n$-\textit{NN classifier}. The universal weak consistency of $k_n$-{NN} was established under the assumptions that $k_n\to\infty$ and $k_n/n\to 0$~\cite{stone1977}. Thereafter, it was shown in~\cite{Devroye1994} that the universal strong consistency holds if we assume further that $k_n/\log n \to \infty$. In this paper, the notion of weak and strong consistency of $k_n$-NN will be under these two respective regimes.

In a general metric space $(X,d_{X})$, the situation is more complicated. Kumari~\cite{Kumari2018} gave an example of a $k_n$-NN classifier on a compact metric space that satisfies the above conditions, but the weak consistency does not hold. To see which additional condition that we might need, let us first define $\ovl{B}(x,r)$ to be the closed ball of radius $r$ centered at $x$. Chaudhuri and Dasgupta~\cite{Chaudhuri2014} showed that, in addition to the assumptions above, if $(X,d_{X})$ is also separable and satisfies the \textit{differentiation condition} for any Borel probability measure $\rho$ and any bounded $\rho$-measurable function $f$:
\begin{equation}\label{eq:diff}
    \lim_{r\downarrow 0} \frac1{\rho(B(x,r))}\int_{B(x,r)} f \ d\rho = f(x),
\end{equation}
 for $\rho$-a.e. $x\in X$, then $k_n$-NN is universally strongly consistent on $(X,d_{X})$. We recommend~\cite{Chen2018} for a recent survey of relevant results.

 The aim of this work is to study the universal consistency of the $k_n$-NN classifier on the Wasserstein space $\W_{p}(X)=(\pp(X),W_p)$; we shall call this the \emph{Wasserstein} $k$-NN. Here, the distance ties are broken by preferring the data points that come earlier. 

 As the main contribution of this work, we show that the $k_n$-NN classifier is \emph{not} universally consistent on $W_p((0,1))$ for any $p\geq 1$. This also implies that, for any $X\subseteq \R^d$ containing a line segment, it is also not universally consistent on $W_p(X)$. In particular, it is not universally consistent on $W_p(\R^d)$ for any $d\geq 1$. Therefore, without any restriction on the Wasserstein space, one should not expect the universal consistency to hold.

 It is then natural to look for some subspaces of $W_p(\R^d)$, on which the universally consistency holds. To this end, we consider the following specific examples: the space of measures supported on an increasing union of uniformly discrete sets for $p\geq 1$, the space of measures supported on a finite set for $p=1$, the space of measures with densities expressed as finite wavelet series for $p=1$, and the space of Gaussian distributions for $p=2$. On these spaces, we show that the $k_n$-NN classifier is universally consistent.

\subsection{Prior work}

There has not been much work on the consistency of the nearest neighbor classifiers on Wasserstein spaces. Nonetheless, a lot of progress has been made on the metric spaces in general. C{\'e}rou and Guyader~\cite{Crou2006} showed that, if the convergence in~\eqref{eq:diff} is in probability, then the $k_n$-NN on any separable metric space is universally weakly consistent. Biau, Bunea and Wegkamp~\cite{Biau2005} proved the universal weak consistency of a modified $k_n$-NN on any separable Hilbert space by exploiting the finite-dimensional truncation. There is also a line of work on a 1-nearest-neighbor-based classifier that is universally strongly consistent on separable metric spaces, even without the differentiation condition~\cite{Hanneke2019,Kontorovich2017}.  

In terms of the differentiation condition~\eqref{eq:diff}, the earliest work is from~\cite{Preiss1979}, who gave an example of a finite measure $\rho$ on a separable infinite dimensional Hilbert space such that the condition does not hold. Later,~\cite{Preiss1983} introduced the notion of $\sigma$-\textit{finite dimension}. He claimed, with only an outline of the proof, that this notion is equivalent to the differentiation condition on separable metric spaces. The proof was then completed in~\cite{Assouad2006}. 

There have been several studies that link the universal consistency of $k_n$-NN to other metric properties. For example, it was proved in~\cite{Assouad2006} and~\cite{Collins2020} that universal strong consistency holds in all metric spaces with $\sigma$-finite \textit{Nagata dimension}~\cite{Nagata1964}. In set-theoretical aspects, it was shown in~\cite{Hanneke2019} and~\cite{Pestov2020} that the universal strong consistency holds in a metric space if the smallest cardinality of its dense subsets is strictly less than real-valued measurable cardinal.

In computational aspects, a series of approximate algorithms have been developed to speed up the nearest-neighbor search in $W_1$.  Kusner, Sun, Kolkin and Weinberger~\cite{kusnerb15} proposed a simple closest-point matching method between two empirical distributions. Atasu and Mittelholzer~\cite{Atasu19} later added capacity constraints to this method, which leads to more accurate estimates that can be computed almost as efficiently. There is an emerging line of works that aim for fast computation using tree-based methods, for example~\cite{Backurs2019} and~\cite{Indyk2003}.

\subsection{The main results}\label{sec:mainthms}

Our main contribution is the following negative result on general Wasserstein spaces:
\begin{Th}\label{th:counter}
    For any $p\geq 1$, the $k_n$-NN classifier is \emph{not} universally consistent on $\mathcal{W}_p((0,1))$.
\end{Th}

This implies that the universal consistency on $\mathcal{W}_p(X)$ does not hold whenever the base metric space $X$ contains a line segment; thus, even when $X$ is a bounded set in $\R^d$, one cannot hope to obtain a positive result on $\mathcal{W}_p(X)$. This demonstrates the vastness of Wasserstein spaces compared to the Euclidean spaces.

Thus, to obtain a positive result, one has to make a ``strong'' restriction on the base metric space, or even on the Wasserstein space itself. For our first positive result, we consider the space of measures with rational mass:
\[
    \mathcal{P}_r(X) = \left\{\sum_{i=1}^k r_i\delta_{x_i}\in\pp(X) \ \Big| \  r_i\in\mathbb{Q}, \quad k\in\mathbb{N} \right\}.
\] 

We also introduce a notion of \emph{$\sigma$-uniformly discrete space}, which is an increasing union of uniformly discrete sets.

\begin{Def}
    A metric space $(X.d)$ is $\sigma$-uniformly discrete if there exists $\{A_n\}_{n\in\mathbb{N}}$ and $\{\Delta_n\}_{n\in\mathbb{N}}$ such that $A_n\subseteq A_{n+1}\subseteq X$ for all $n\in\mathbb{N}$ and $d(x,y) \geq \Delta_n>0$ for any distinct $x,y\in A_n$.
\end{Def}

We are now ready to state the first positive result.

\begin{Th}\label{thm:main1}
    Suppose that a metric space $(X,d)$ is $\sigma$-uniformly discrete. Then, for any $p\geq 1$, the $k_n$-NN classifier is universally consistent on $(\mathcal{P}_r(X),W_p)$. 
\end{Th}

For $(X,d)=(\mathbb{Q}^d,\|\cdot\|_2)$, we can express each uniformly discrete subset via the factorial system:
\[
    A_n = \left\{\left(\frac{a_1}{n!},\ldots,\frac{a_d}{n!}\right) \ \Big| \ (a_1,\ldots,a_d)\in\mathbb{Z}^d \right\},
\] 
from which we can take $\Delta_n =\frac{1}{n!}$. This leads to the following consistency result on a dense subset of $\W_p(\R^d)$:

\begin{Cor}
    The $k_n$-NN classifier is universally consistent on $(\mathcal{P}_r(\mathbb{Q}^d),W_p)$ for all $d\in\mathbb{N}$ and all $p\geq 1$. 
\end{Cor}

In contrast to the previous results, which hold for all $p\geq 1$, the next positive results are proved only for $p=1$ or $p=2$; this is because the proofs rely on the geodesic structure of the Wasserstein space for those values of $p$ (see Section{~\ref{sec:wpc}} below).

Specifically, we prove the universal consistency of the Wasserstein $k_n$-NN on measures supported on a finite metric space.

\begin{Th}\label{thm:main0}
    Let $(X,d)$ be a finite metric space. Then the $k_n$-NN classifier is universally consistent on $(\mathcal{P}(X),W_1)$. 
\end{Th}

This gives theoretical support, for instance, to $k$-NN classification of color histograms (where $X=\{0,1,\ldots,255\})$ or document histograms (where $X$ consists of all words in the vocabulary).

For the next result, we consider the family of Gaussian measures under the $2$-Wasserstein distance. For $m\in\mathbb{R}^d$ and $\Sigma\in\operatorname{Sym}^{+}(d)$, let $\mu_{m,\Sigma}$ be the Gaussian measure with mean $m$ and covariance matrix $\Sigma$. Denote the family of $d$-dimensional Gaussian measures by:       \[\pp_{G}(d) = \left\{\mu_{m,\Sigma}\in\pp(\R^d) \ | \ m\in\R^d \ \text{ and } \ \Sigma\in\operatorname{Sym}^{+}(d) \right\}. \]
We will show that, under $W_2$, the $k_n$-NN classification of measures in $\mathcal{P}_G(d)$ is universally consistent.
     \begin{Th}\label{thm:main3}
         The $k_n$-NN classifier is universally consistent on $(\mathcal{P}_G(d),W_2)$.
\end{Th}
Note that the Theorem follows immediately from the fact that any the Lebesgue differentiation theorem holds on any separable $C^2$-Riemannian manifold~\mbox{\cite[Section 2.8]{Federer1996}}. We provide here an alternative proof, which might be of independent interest.

Next, we consider probability densities in terms of wavelet expansion. Let $\phi,\psi\in L_2(\R)$ be wavelet functions, $\phi_{\ell k}=2^{\ell/2}\phi(2^{\ell}x-k)$ and $\psi_{jk}=2^{j/2}\phi(2^jx-k)$ for $j\geq \ell$. We consider probability densities in $L_p([0,1])$ in form of finite wavelet series
\begin{equation} \label{eq:wl}
    \begin{split}
    f &= \sum_{k\in\mathbb{Z}}\alpha_{\ell k}\phi_{\ell k} + \sum_{j= \ell}^{N}\sum_{k\in\mathbb{Z}}\beta_{jk}\psi_{jk} \\
g &= \sum_{k\in\mathbb{Z}}\alpha'_{\ell k}\phi_{\ell k} + \sum_{j= \ell}^{N}\sum_{k\in\mathbb{Z}}\beta'_{jk}\psi_{jk}.
\end{split}
\end{equation}
These densities arise from nonparametric density estimation~\cite{Donoho1996} with applications in signal classification~\cite{Pah2003,Szczuka2001}. Here, we make the following assumptions on $\phi$ and $\psi$:
\begin{itemize}
    \item $\phi$ and $\psi$ are compactly supported. Thus, for $x\in[0,1]$, there exists $K_0,K_{j}\in\mathbb{N}$ such that $\phi_{\ell k}(x)=0$ for all $|k|>K_0$ and $\psi_{jk}(x)=0$ for all $|k|>K_j$.
    \item All constant functions lie in the span of $\{\phi_{\ell k}\}_{k\in \mathbb{Z}}$.
    \item $\phi$ and $\psi$ are continuously differentiable.
    \item $\|\psi_{jk}\|_{L_1[0,1]} = C2^{-\frac{1}{2}j}$ for some universal constant $C$.
\end{itemize}

Examples of wavelets that satisfy these assumptions include Daubechies wavelets~\cite{Daubechies1988,Cohen1993}. 
\begin{Th}\label{thm:main2}
    Let $\mathcal{V}_N$ be the set of probability measures with densities in the form of~\eqref{eq:wl}. Then the $k_n$-NN classifier is universally consistent on $(\mathcal{V}_N,W_1)$. 
\end{Th}

We will introduce the main ingredients that allows us to turn the universal consistency into a geometrical problem (Section{~\ref{sec:prelim}}). We then proceed to prove the main negative result (Theorem{~\ref{th:counter}}) in Section{~\ref{sec:counter}}. Next, we prove the first positive result (Theorem{~\ref{thm:main1}}) in Section{~\ref{sec:sigma}}. We then introduce the notions of geodesics in a metric space and weakly positively curved spaces in Section{~\ref{sec:wpc}} which allows us to prove the remaining positive results (Section{~\ref{sec:th1}},{~\ref{sec:gaussian}} and{~\ref{sec:wavelet}}).

\section{Notations}\label{sec:notations}
We use the following notations throughout this paper: $\mathbf{1}_A$ is the indicator function of a set $A$. $\delta_x$ is the Dirac measure at $x$. $\text{supp}(\mu)$ is the support of a measure $\mu$. $B(x,r)$ and $\ovl{B}(x,r)$ are the open ball and the closed ball of radius $r$ centered at $x$, respectively. $\operatorname{Sym}(d)$ is the set of all $d\times d$ real symmetric matrices. $\operatorname{Sym}^+(d)$ is the set of all $d\times d$ real positive-semidefinite symmetric matrices. $\operatorname{Sym}^{++}(d)$ is the set of all $d\times d$ real positive-definite symmetric matrices. Let $\mathcal{B}=\{A_i\}_{i\in I}$ be a family of subsets of $X$. The \emph{multiplicity} of $\mathcal{B}$ is defined by the infimum of all $\beta$ that satisfies $\sum_{i\in I}\mathbf{1}_{A_i}(x)\leq\beta$ for all $x\in X$.

\section{Preliminary results}\label{sec:prelim}

We will follow the consistency results in~\cite{Chaudhuri2014} which hold under the following regime:
\begin{Def}
    We say that the $k_n$-NN classifier is \emph{universally consistent} on a metric space $(X,d)$ if it satisfies the following conditions:
\begin{itemize}
    \item If $k_n\to\infty$ and $k_n/n\to 0$, then it is universally weakly consistent on $X$.  
    \item If in addition $k_n/\log n\to \infty$, then it is universally strongly consistent on $X$.  
\end{itemize}
\end{Def}

The following theorem from~\cite{Chaudhuri2014} connects the differentiation condition~\eqref{eq:diff} to the universal consistency of the $k_n$-NN classifier on separable metric spaces.
\begin{Th}\label{thm:cstn}
    Let $(X,d)$ be a separable metric space such that~\eqref{eq:diff} holds $\rho$-a.e. $x\in X$ for all Borel probability measure $\rho$ and all bounded measurable function $f$. Then the $k_n$-NN classifier is universally consistent on $X$.
\end{Th}

The main task is now to show that $\W_p(X)$ satisfies the differentiation condition. In the context of Theorem~\ref{thm:cstn}, this seems rather difficult as we have to show that~\eqref{eq:diff} holds for all measure $\mu$. Fortunately, this condition is equivalent to a purely topological one. First, let us introduce the notion of \emph{metric dimension} 

\begin{Def}
    Given $s>0$, we say that closed balls $(\ovl{B}(x_i,r_i))_{1\leq i \leq m}$ in a metric space are \emph{disconnected at scale} $s$ if $r_1,\ldots,r_m\in (0,s)$ and $x_i\notin \ovl{B}(x_j,r_j)$ for all $i\not= j$.

    If such condition holds for all $s>0$, then they are \emph{disconnected}.
\end{Def}

\begin{Def}\label{def:dim}
    Let $(X,d)$ be a metric space and $\beta\in \mathbb{N}$. A set $Y\subseteq X$ has \emph{metric dimension} $\beta$ \emph{at scale} $s$ in $X$, or $\dim^s_{X}(Y)=\beta$, if $\beta$ is the smallest positive integer such that, for any family of disconnected closed balls $(B_i)_{1\leq i \leq m}$ at scale $s$ whose centers belong to $Y$, their multiplicity is at most $\beta$. In other words, 
\[
    \sum_{i=1}^{m}\mathbf{1}_{B_i}(x)\leq \beta
\] 
for all $x\in X$. If no such $\beta$ exists, we assign $\dim^s_{X}(Y)=\infty$.

If $\dim^s_{F}(Y) = \beta$ for all $s>0$, we simply write $\dim_{X}(Y)=\beta$.
\end{Def}

In other words, $\dim_{X}(Y)=\beta$ if any point in $X$ can belong to at most $\beta$ disconnected closed balls whose centers are contained in $Y$. It is difficult to compute the metric dimension in general, but we will only be concerned with whether or not it is finite. 

Unsurprisingly, Euclidean spaces have finite metric dimension.
\begin{Ex}\label{prop:rn}
   For any $d\geq 1$,
   \begin{equation}
       \dim_{\mathbb{R}^d}(\R^d)\leq 3^d-1. \label{eq:rn}
   \end{equation}
\end{Ex}
\begin{proof}
    Consider a family of disconnected closed balls $(\ovl{B}(x_i,r_i))_{1\leq i \leq m}$ in $\R^d$ whose intersection is nonempty. It suffices to show that $m\leq 3^d-1$ 

    For any $a,b\in \R^d$, we denote by $\ell(a,b)$ the line that passes through $a$ and $b$. Given $x\in \bigcap_{i} \ovl{B}(x_i,r_i)$ and $r>0$, let us define $y_i=\partial\ovl{B}(x,r)\cap \ell(x, x_i)$. First, we will show that $d(y_i,y_j)\geq r$ for all pairs of distinct $i$ and $j$. This is trivial when $x,y_i$ and $y_j$ are collinear, so we shall assume that this is not the case. We also assume without loss of generality that $d(x,x_i)\leq d(x,x_j)$. There is a point $z_j\in \ell(x,x_j)$ that makes $\ell(y_i,z_j)$ parallel to $\ell(x_i,x_j)$. Since $d(x_i,x_j)\geq d(x,x_j)$, we also have $d(y_i,z_j)\geq d(x,z_j)$. This observation and the triangle inequality yield
    \[
        d(y_i,y_j)  \geq d(y_i,z_j)-d(z_j,y_j) \geq d(x,z_j)-d(z_j,y_j) = d(x,y_j) = r,
    \] 
    as claimed. This implies that the balls $B(y_i,\frac{r}{2})$ are disjoint. Let $v_d$ be the volume of the unit ball in $\R^d$. It follows that
    \begin{align*}
        \bigcup_{i=1}^{m} B(y_i,\tfrac{r}{2}) &\subset B(x,\tfrac{3r}{2})\setminus B(x,\tfrac{r}{2}) \\
        mv_d(\tfrac{r}{2})^d &\leq v_d(\tfrac{3r}{2})^d-v_d(\tfrac{r}{2})^d \\
        m           &\leq 3^d-1,
    \end{align*}
   as desired.
\end{proof}

As we can see, the proof relies on the ratio-preserving property of the homothety in the Euclidean space. As the bound in~\eqref{eq:rn} grows with the dimension, this notion is generally not applicable to infinite dimensional spaces. This motivates the following definition:

\begin{Def}
    A metric space $(X,d)$ has $\sigma$-\emph{finite metric dimension} if there is a countable family $\{Y_n\}_{n\in\mathbb{N}}$ of subsets of $X$ such that $\dim^{s_n}_{X}(Y_n)<\infty$ for some $s_n>0$ and 
    \begin{equation}\label{eq:union}
        X = \bigcup_{i=1}^{\infty}Y_n.
    \end{equation} 
\end{Def}

For example, the space of square-summable infinite sequences $d^2$ with the usual metric has $\sigma$-finite metric dimension. The link between this notion and the differentiation condition lies in the following result from Assouad and Quentin de Gromard~\cite{Assouad2006}. The proof of this Theorem is provided in Appendix~\ref{sec:proofthm}.

\begin{Th}\label{thm:equiv}
    Let $(X,d)$ be a separable metric space with $\sigma$-finite metric dimension. Then the differentiation condition~\eqref{eq:diff} holds $\rho$-a.e. $x\in X$ for any finite Borel measure $\rho$ and any bounded $\rho$-measurable function $f$.
  \end{Th}
  Note that the converse holds for complete metric spaces, as Kumari~\cite{Kumari2018} recently proved that any complete separable metric space that satisfies the differentiation condition also has $\sigma$-finite metric dimension.

Thus, to obtain universal consistency, it suffices to show that $X$ has $\sigma$-finite metric dimension. The completeness and separability requirement in Theorem~\ref{thm:equiv} can be achieved for a Wasserstein space given that the base metric space is complete and separable. A constructive proof is due to~\cite{Bolley2008}.
\begin{Th}\label{thm:bolley}
    If a metric space $X$ is complete and separable, then $\W_p(X)$ is also complete and separable.
\end{Th}

\section{The $k_n$-NN classifier is not universally consistent on $\mathcal{W}_p((0,1))$}\label{sec:counter}

\begin{proof}[Proof of Theorem~\ref{th:counter}]

    We will construct a Borel probability measure $\rho$ on $\mathcal{W}_p((0,1))$, and for any $\mu\in \mathcal{P}((0,1))$ a conditional probability $\eta(\mu)=\mathbb{P}(Y=1\mid\mu)$ so that the $k_n$-NN classifier is not weakly consistent.


For any $p\geq 1$, the $p$-Wasserstein distance between $\mu,\nu\in\mathcal{P}(\R)$ is given by
\[
    W^p_p(\mu,\nu)=\int_0^1\lvert f_{\mu}(x)-f_{\nu}(x) \rvert^p \ dx,
\] 
where $f_{\mu}$ and $f_{\nu}$ are the generalized quantile functions (GQF): $f_{\mu}(p) = \inf\{x\in\R\cup\{-\infty\} \mid p\leq F_{\mu}(x)\}$ where $F_{\mu}$ is the cumulative distribution function of $\mu$. Note that GQF functions are non-decreasing and left-continuous, and any function with these properties gives rise to a probability measure. 

With this in mind, we construct a family of GQF functions as follows: let $(a_i)_{i\in\N}$ be a strictly increasing sequence of positive numbers satisfying $a_i<1$ for all $i\in\N$ and $\sum_{i=1}^{\infty}a^{p}_i/2^i<\infty$. Define $I_i=[1-1/2^{i-1},1-1/2^i)$ for $i\in\N$; thus $\bigcup_{i\in\N}I_i = [0,1)$. Define a staircase function $f_0:[0,1)\to\R$ by:
\[
    f_0 = \sum_{i=1}^{\infty} a_i\mathbf{1}_{I_i}.
\] 
For $m\in\N$, define $f_m:[0,1)\to\R$ to be the same as $f_0$, except the $m$-th step size is widen to $a_{m+1}$, that is,
\[
    f_m =\sum_{i=1}^{m-1}a_i\mathbf{1}_{I_i} + a_{m+1}\mathbf{1}_{I_m}+\sum_{i=m+1}^{\infty}a_i\mathbf{1}_i.
\] 
Note that for any $m\geq 0$, the measure $\mu_m$ associated with $f_m$ is supported in $\{a_i \mid i\in\N\}\subset (0,1)$. Thus $\mu_m\in\mathcal{P}((0,1))$.

Notice that, for any distinct $j,m\geq 1$, $f_j$ and $f_m$ differ on $I_j$ and $I_m$, while $f_j$ and $f_0$ differ only on $I_j$. Therefore, $W_p(\mu_j,\mu_m)>W_p(\mu_j,\mu_0)$, and similarly, $W_p(\mu_j,\mu_m)>W_p(\mu_m,\mu_0)$. Consequently, the set 
\[U=\{\mu_m\mid m\in\N\cup \{0\}\}\] 
has infinite metric dimension at any scale under $W_p$, since for any $s>0$, there exists $M\in\N$ such that $W_p(\mu_m,\mu_0)<s$ for any $m\geq M$, and any two closed balls in $\{\overline{B}(\mu_m,W_p(\mu_m,\mu_0))\}_{m\in\N}$ intersect at a single point $\mu_0$.

We now define a Borel measure $\rho$ on $\mathcal{W}_p((0,1))$ as follows: $\rho(\{\mu_0\})=1/2$, $\rho(\{\mu_m\})=1/2^{m+1}$ for all $m\geq 1$ and $\mu(\mathcal{P}(\R)\setminus U)=0$. We give all $\mu_m$ deterministic labels: $Y(\mu_0)=1$ and $Y(\mu_m)=0$ for all $m\geq 1$. Let $D_n$ be a sample of $n$ measures under $\rho$ and choose $k_n=\sqrt{n}$. Let $X_n$ be the random variable of number of $\mu_0$'s in $D_n$. A key observation is that the classification of $\mu_m$ for any $m\geq 1$ will be wrong if $X_n> k_n=\sqrt{n}$.

Thus we are interested in the events of $D_n$ in which there are sufficient numbers of $\mu_0$. Since $X_n\sim\operatorname{Binomial}(n,1/2)$, we can utilize the Hoeffding's inequality:
\begin{align*}
    \mathbb{P}(X_n>\sqrt{n}) &= 1-\mathbb{P}(X\leq \sqrt{n}) \\
                                  &\geq 1-\exp\left(-\frac{c(n/2-\sqrt{n})^2}{n}\right) \\
                                  &= 1-\exp\left(-c\left(\frac{\sqrt{n}}{2}-1\right)^2\right),
\end{align*}
for some constant $c>0$. Let $\mu$ be a sample from $U$ under $\rho$ and $\widehat{Y}_n(\mu)$ be the classification of $\mu$ using the nearest neighbors in $D_n$. As the classification is incorrect if and only if $\mu\not=\mu_0$, we have that 
\begin{equation*}
    \lim_{n\to\infty}\mathbb{E}\left[\mathbb{P}(\widehat{Y}_n(\mu)\not= Y(\mu)|D_n)\right]\geq \lim_{n\to\infty}\mathbb{P}(X_n>\sqrt{n})\mathbb{P}(\mu\not=\mu_0)\geq \frac{1}{2}\lim_{n\to\infty}\mathbb{P}(X_n>\sqrt{n}) = \frac{1}{2}.
\end{equation*}
However, the Bayes risk is zero since the labels are deterministic. We conclude that the $k_n$-NN classifier is not weakly consistent for the measure $\rho$ on $\mathcal{W}_p((0,1))$.

\end{proof}

\section{Universal consistency when the base space is $\sigma$-uniformly discrete}\label{sec:sigma}

We prove here the first positive result. The main idea is that, whenever $(X,d)$ is $\sigma$-uniformly discrete, the metric space $(\mathcal{P}_r(X),W_p)$ has $\sigma$-finite metric dimension, which implies that the $k_n$-NN is universally consistent on $(\mathcal{P}_r(X),W_p)$.\\
\begin{proof}[Proof of Theorem~\ref{thm:main1}]
    Since $(X,d)$ is $\sigma$-uniformly discrete, there exists $\{A_n\}_{n\in\mathbb{N}}$ and $\{\Delta_n\}_{n\in\mathbb{N}}$ such that $A_n\subseteq A_{n+1}\subseteq X$ for all $n\in\mathbb{N}$ and $d(x,y) \geq \Delta_n>0$ for any distinct $x,y\in A_n$. Recall that
\[
    \mathcal{P}_r(X) = \left\{\sum_{i=1}^k r_i\delta_{x_i}\in\pp(X) \ \Big| \  r_i\in\mathbb{Q}, \quad k\in\mathbb{N} \right\}.
\] 
We can write $\mathcal{P}_r(X)=\bigcup_{n}\mathcal{A}_n$ where
\[\mathcal{A}_{n} = \left\{ \sum_{i=1}^{k}\frac{a_i}{n!}\delta_{x_i}\in\pp_r(A_n) \ \Big| \ 0\leq a_i\leq n!, \quad k\in\mathbb{N}  \right\}.\]
As $A_n\subset A_{n+1}$, we have $\mathcal{A}_n\subset\mathcal{A}_{n+1}$ for all $n\in\mathbb{N}$. In addition, for any distinct $\mu,\nu\in\mathcal{A}_N$, at least a mass of $\frac{1}{n!}$ must be transported by at minimum distance of $\Delta_n$, yielding $W_p(\mu,\nu)\geq \frac{\Delta_n}{n!}$. It follows that, if we choose $s_n=\frac{\Delta_n}{2n!}$, the family of closed balls $\{\ovl{B}(\mu,r_{\mu})\}_{\mu\in\mathcal{A}_n}$ where $r_{\mu}\in (0,s_n)$ are mutually disjoint. In other words, any disconnected family of closed balls centered in $\mathcal{A}_n$ at scale $s_n$ has zero multiplicity. Hence, $\pp_r(X)$ has $\sigma$-finite metric dimension and so the $k_n$-NN classifier is universally consistent on $(\pp_r(X),W_p)$.
\end{proof}

\section{Weakly positively curved spaces}\label{sec:wpc}

Going back to the proof of Example~\ref{prop:rn}, we see that the proof of the upper bound of $\dim_{\R^d}(\R^d)$ relies on its underlying geometry, specifically, its similarity-preserving homothety. Some of our results can be proved in the same spirit as this example, where the Euclidean lines are replaced by a similar notion in a curved space.

\begin{Def}
    In a metric space $(X,d)$, a curve $\{ x_{1,2}^t\in X : t\in [0,1] \}$ is a \emph{constant speed geodesic} between $x_1$ and $x_2\in X$ if for any $s,t\in [0,1]$,
    \begin{equation}\label{eq:csspd}
        d(x_{1,2}^s,x_{1,2}^t) = |t-s|d(x_1,x_2).
    \end{equation} 
\end{Def}

In the case of $\W_1(\R^n)$, it is easy to check that $\mu^t =(1-t)\mu_1+t\mu_2$ is a constant speed geodesic from $\mu_1$ to $\mu_2$: for any $s,t\in [0,1]$ 
\begin{align*}
    W_1(\mu^s,\mu^t) &= \sup_{\|\nabla f\|_{\infty}\leq 1}\int f \ (d \mu^s - d\mu^t) \\
                     &= \sup_{\|\nabla f\|_{\infty} \leq 1}(t-s)\int f \ (d \mu_1 - d\mu_2)  \\
                                                                                             &= |t-s| W_1(\mu_1,\mu_2).
\end{align*} 

We will be studying some geometrical properties of $\W_p(\R^d)$ through these geodesics. Specifically, the following inequality will be used to measure the curvature of geodesic triangles.  

\begin{Def}
    A metric space $(X,d)$ is a \emph{weakly positively curved} space (\WPC space) if for any $x_1,x_2,x_3\in X$, there is a constant speed geodesic $x^t_{1,2}$ connecting $x_1$ and $x_2$ and $x^t_{1,3}$ connecting $x_1$ and $x_3$ that satisfy the following \emph{comparison inequality}:
    \begin{equation}\label{eq:pcineq}
        d(x^t_{1,2},x^t_{1,3}) \geq t d(x_2,x_3).
    \end{equation}
    for any $t\in [0,1]$
\end{Def}

\begin{figure}[tpb]
    \centering
    \includegraphics[width=0.3\linewidth]{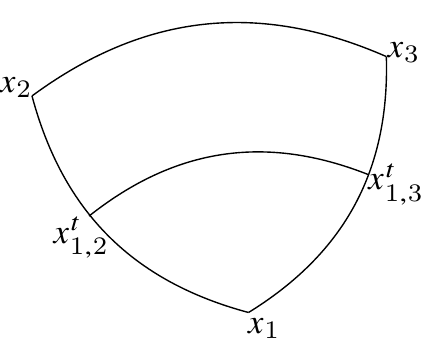}
    \caption{A triangle in a WPC space.}
    \label{fig:wpc}
\end{figure}

Roughly speaking, a metric space is a \WPC space if the sides of every geodesic triangle are curved outward. It is a weaker notion of \emph{positively curved} space (\PC space) defined in Lemma~\ref{th:w2pc} below. It turns out that both $\mathcal{W}_1(\R^d)$ and $\mathcal{W}_2(\R^d)$ are \WPC spaces.

\begin{Th}\label{thm:w1wpc}
    $\W_1(\R^d)$ is a \WPC space. 
\end{Th}
\begin{proof}
Let $\mu_1,\mu_2,\mu_3\in \pp(\R^d)$. Then for the geodesics $\mu^t_{1,2}=(1-t)\mu_1+t\mu_2$ and $\mu^t_{1,3}=(1-t)\mu_1+t\mu_3$, we have
\begin{align*}
        W_1(\mu^t_{1,2},\mu^t_{1,3}) &= \sup_{\|\nabla f\|_{\infty}\leq 1}\int f \ (d \mu_{1,2}^t - d\mu_{1,3}^t) \\
                     &= \sup_{\|\nabla f\|_{\infty} \leq 1}t\int f \ (d \mu_2 - d\mu_3)  \\
                                                                                             &= t W_1(\mu_2,\mu_3).
\end{align*}
\end{proof}

\begin{Th}\label{thm:w2wpc}
    $\W_2(\R^d)$ is a {\WPC} space.
\end{Th}
\begin{proof}
    We start with the fact that $\W_2(\R^d)$ satisfies a stronger notion than \WPC~\cite[Section 7.3]{Ambrosio2005}:
\begin{itemize}
    \item[]
        \vspace*{-6mm}
        \begin{Lemma}\label{th:w2pc}
            $\W_2(\R^d)$ is a \emph{positively curved} space ({\PC} space). In other words, for any $\mu_1,\mu_2,\mu_3\in \pp_2(R^n)$ and any constant speed geodesic $\mu_{1,2}^t$ from $\mu_1$ to $\mu_2$, we have the following inequality:
    \begin{equation}\label{eqn:pc1}
        W_2^2 (\mu_{1,2}^t,\mu_3) \geq (1-t)W_2^2(\mu_1,\mu_3)+tW_2^2(\mu_2,\mu_3)-t(1-t)W_2^2(\mu_1,\mu_2).
    \end{equation}
\end{Lemma}
\end{itemize}

For more details on \PC spaces and their cone structures, see~\cite[Chapter 12.3]{Ambrosio2005}. It turns out that any \PC space is also a \WPC space, as we will show below.

Let $\mu_1,\mu_2,\mu_3$ and $\mu_{1,2}^t$ be as in Lemma~\ref{th:w2pc} and $\mu_{1,3}^t$ be a constant speed geodesic from $\mu_1$ to $\mu_3$. Applying~\eqref{eqn:pc1} to the measures $\mu_1,\mu_3,\mu_{1,2}^t$, we obtain
    \begin{align*}
         W_2^2(\mu_{1,3}^t,\mu_{1,2}^t) &\geq (1-t)W_2^2(\mu_1,\mu_{1,2}^t)+tW_2^2(\mu_3,\mu_{1,2}^t)-t(1-t)W_2^2(\mu_1,\mu_3) \\
                                       &\geq (1-t)W_2^2(\mu_1,\mu_{1,2}^t) \\
                                       & \qquad +t[(1-t)W_2^2(\mu_1,\mu_3)+tW_2^2(\mu_2,\mu_3)-t(1-t)W_2^2(\mu_1,\mu_2)] \\
                                       & \qquad -t(1-t)W_2^2(\mu_1,\mu_3) \\
                                       &= (1-t)W_2^2(\mu_1,\mu_{1,2}^t) + t^2_2W_2^2(\mu_2,\mu_3)-t^2(1-t)W_2^2(\mu_1,\mu_2) \\
                                       &= t^2_2W_2^2(\mu_2,\mu_3),
    \end{align*}
    where we used $W_2(\mu_1,\mu^t_{1,2})=tW_2(\mu_1,\mu_2)$ in the last step.
\end{proof}

The following lemma is the main tool that will help us prove $\sigma$-finite dimensionality of metric spaces in our interest by linking them back to the Euclidean spaces (Example~\ref{prop:rn}). 

\begin{figure}[tpb]
    \centering
    \includegraphics[width=0.8\linewidth]{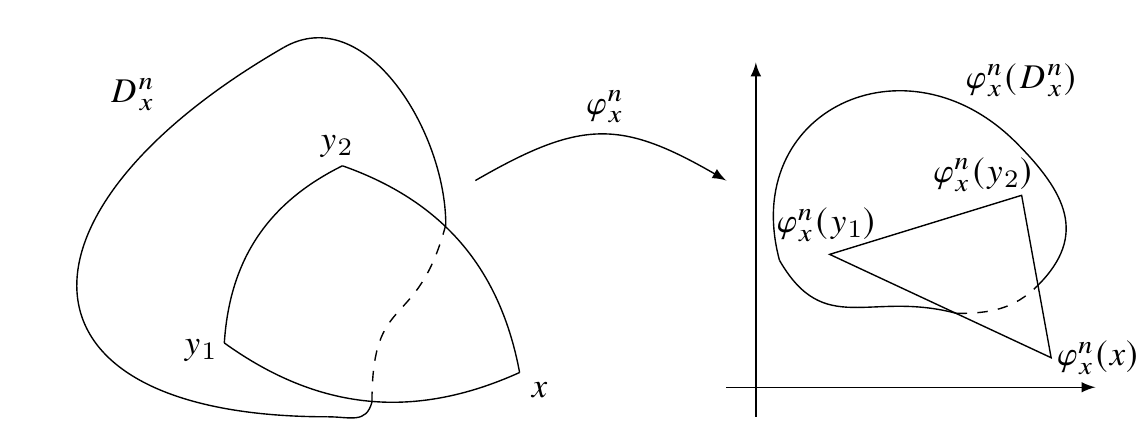}
    \caption{A schematic picture of the setup in Lemma~\ref{lemma:helper}.}
    \label{fig:}
\end{figure}
\begin{Lemma}\label{lemma:helper}
    Let $(X,d)$ be a complete separable \WPC space where $X=\cup_{n\in \mathbb{N}}A_n$. For each $x\in X$ and each $y\in A_n$, let $\{y_x^t\}_{t\in [0,1]}$ be a specific choice of geodesic from $x$ to $y$. With this notion, we define a cone emanating from $x$ to a set $D\subset X$:
\[
    \mathcal{G}_x(D) = \{ y_x^t \ | \ y\in D, \quad  t\in [0,1] \}. 
\] 
Suppose that for each $n\in \N$, there exists $s_n>0$ with the following property: for any $x\in X$ such that $D^n_{x}=B(x,s_n)\cap A_n\not=\emptyset$, there exists a function $\varphi_{x}^n:\mathcal{G}_x(D^n_x)\to \R^{d_n}$, for some constant $d_n$, such that the following inequalities hold for all $y_1,y_2\in \mathcal{G}_x(D^n_x)$:
    \begin{align}
        d(x,y_1) &\geq c_n\|\varphi_{x}^n(x)-\varphi_{x}^n(y_1)\|_2 \label{eq:rnlower}\\
        d(y_1,y_2) &\leq C_n\|\varphi_{x}^n(y_1)-\varphi_{x}^n(y_2)\|_2, \label{eq:rnupper}
    \end{align}
    for some constants $c_n,C_n>0$ independent of $x$. Then $(X,d)$ has $\sigma$-finite metric dimension. 
\end{Lemma}
\begin{proof}
    Let $\{B(y_i,r_i)\}_{1\leq i \leq m}$ be a disconnected family of closed balls centered in $D^n_x$ such that $r_i<s_n$ for all $i\in\N$, and assume that $x\in\bigcap_{i\in\N}B(y_i,r_i)$. Thus $B(x,s_n)\cap A_n\not=\emptyset$, so there exists a function $\varphi_x^n$ that satisfies~{\ref{eq:rnlower}} and~{\ref{eq:rnlower}}. Let $\{y_i^t\}_{t\in[0,1]}$ be the geodesic between $x$ and $y_i$. From the comparison inequality~\eqref{eq:pcineq}, we have  
\begin{equation}
    \begin{split}
    d(y_{i}^t,y_j^t) &= t d(y_i,y_j) \\
                         &\geq t\max \{d(x,y_i),d(x,y_j)\}. \label{eq:XcompW}
    \end{split}
\end{equation} 
Denote $r_i =d(x,y_i)$ and let $r$ be the minimum of all the $r_i$'s. With $\alpha_i=r/r_i$, it follows from the property of constant speed geodesics that
\begin{equation}\label{eq:cspd}
d(x,y_i^{\alpha_i})= r.
\end{equation} 
In other words, $y_i^{\alpha_i}$ is the projection of $y_i$ on the sphere of radius $r$ centered at $x$. Focusing on each pair of $i$ and $j$, we assume without loss of generality that $r_i\leq r_j$. The triangle inequality and~\eqref{eq:XcompW} yield
\begin{equation}\label{eq:projr}
        \begin{split}
        d(y^{\alpha_i}_i,y^{\alpha_j}_j)&\geq d(y^{\alpha_i}_i,y^{\alpha_i}_j)-d(y^{\alpha_i}_j,y^{\alpha_j}_j)\\
                                    &\geq \alpha_i d(y_i,y_j)-(\alpha_i-\alpha_j)d(x,y_j) \\
                                    &> \alpha_i d(x,y_j)-(\alpha_i-\alpha_j)d(x,y_j) \\
                                    &= \alpha_j d(x,y_j)\\
                                    &= r.
    \end{split}
\end{equation}
Using~\eqref{eq:rnlower} and~\eqref{eq:rnupper}, 
\[
    \|\varphi_{x}^n(x)-\varphi_{x}^n(y_i^{\alpha_i})\|_2 \leq c_n^{-1}d(x,y_i^{\alpha_i})=c_n^{-1}r
\] 
and
\[
    \|\varphi_{x}^n(y_i^{\alpha_i})-\varphi_{x}^n(y_j^{\alpha_j})\|_2 \geq C_n^{-1}d(y_i^{\alpha_i},y_j^{\alpha_j})\geq C_n^{-1}r.
\] 

We thus have a packing of points $\{\varphi_{x}^n(y_i^{\alpha_i})\}_{1\leq i \leq m}$ inside a closed ball $\ovl{B}(\varphi^n_x(x),c_n^{-1}r)$ which are at least $C_n^{-1}r$ apart from each other. In other words, the enlarged ball $\ovl{B}(\varphi_{x}^n(x),c_n^{-1}r+C_n^{-1}r/2)$ contains all $m$ disjoint balls $B(\varphi_{x}^n(y_i^{\alpha_i}),C_n^{-1}r/2)$. Hence, it must be the case that
    \[
        m\leq \left(\frac{c_n^{-1}r}{C_n^{-1}r/2}+1 \right)^{d_n} = \left(\frac{2C_n}{c_n} +1 \right)^{d_n}.
    \] 
    In particular, $\dim^{s_n}_{X}(A_n)$ is finite and independent of $r$, giving us the conclusion that $k_n$-NN classifier is universally consistent on $X$.
\end{proof}

\section{Universal consistency: other examples}

\subsection{Finitely supported measures}\label{sec:th1}

\begin{proof}[Proof of Theorem~\ref{thm:main0}]
    Writing $X=\{x_1,\ldots,x_d\}$, we construct a map $\varphi:\pp(X)\to \R^{d}$  as follows: 
\[
    \varphi\Big(\sum_{i=1}^{d} v_i\delta_{x_i}\Big)=(v_1,\ldots,v_{d}).
\]

The special thing about the $W_1$ metric is that, given any $\mu, \nu\in\pp(X)$, each measure in the geodesic $\{(1-t)\mu+t\nu\}_{t\in[0,1]}$ is also supported on $X$. Therefore, if we fix $\mu=\sum_{i=1}^{d} a_i\delta_{x_i}$ and let $\nu_1, \nu_2$ be any measures along two different geodesics starting from $\mu$, then we can write $\nu_1=\sum_{i=1}^{d} b_i\delta_{x_i}$ and $\nu_2=\sum_{i=1}^{d} c_i\delta_{x_i}$. The optimal transport from $\nu_1$ to $\nu_2$ must transfer the mass difference at $x_i$, which is $|b_i-c_i|$, by not more than $M=\max_{i,j}d(x_i,x_j)$. This gives us an upper bound 
    \[
        W_1(\nu_1,\nu_2) \leq M \sum_{i=1}^{d}|b_i-c_i|  \leq {d}^{\frac{1}{2}}M \Big[\sum_{i=1}^{d}|b_i-c_i|^2\Big]^{\frac{1}{2}}\leq d^{\frac1{2}}M\|\varphi(\nu_1)-\varphi(\nu_2)\|_2.
    \]
    On the other hand, the optimal transport from $\mu$ to $\nu_1$ must transfer a mass of size $|a_i-b_i|$ by at least $\delta=\min_{i\not= j}d(x_i,x_j)$. Therefore, 
    \[
        W_1(\mu,\nu_1) \geq \sum_{i=1}^{d}\delta|a_i-b_i| \geq \delta\Big[\sum_{i=1}^{d}(a_i-b_i)^2\Big]^{\frac1{2}} = \delta\|\varphi(\mu)-\varphi(\nu_1)\|_2.
    \] 
    Thus, Lemma~\ref{lemma:helper} applies and we have that $k_n$-NN classifier is universally consistent on $(\pp(X),W_1)$.
\end{proof}

\subsection{Gaussian measures}\label{sec:gaussian}

    Before proving the main theorem, we review the Riemannian geometry of Gaussian measures (see~\cite{TAKATSU2012,Malag2018,Bhatia2019} for complete treatments of the subject). The differential structure over $\mathcal{P}_G(d)$ is given by:
    \begin{equation*}
        \mathcal{P}_G(d)\to \R^d\times \operatorname{Sym}^+(d), \qquad \mu_{m,\Sigma}\mapsto (m,\Sigma).
    \end{equation*}
    Given $\mu_1=\mu_{m_1,\Sigma_1}$ and $\mu_2=\mu_{m_2,\Sigma_2}$. The 2-Wasserstein distance between $\mu_1$ and $\mu_2$ is given by~\cite{Dowson1982}:
    \begin{equation}\label{eq:w2g1}
        W_2^2(\mu_1,\mu_2) = \|m_1-m_2\|_2^2+\text{Tr}(\Sigma_1+\Sigma_2-2(\Sigma_1^{1/2}\Sigma_2\Sigma_1^{1/2})^{1/2}).
    \end{equation}
    Notice that~\eqref{eq:w2g1} already contains the Euclidean distance between the means, thus we may assume hereafter that $m_1=m_2=0$. In this view, we denote $\mu_{\Sigma}=\mu_{0,\Sigma}$ and $\mathcal{P}_{G}^0(d)=\{\mu_{\Sigma}\mid \Sigma\in\operatorname{Sym}^+(d)\}$.

    For any $\mu=\mu_{\Sigma}\in \mathcal{P}_G^{0}(d)$ and $X,Y\in T_{\mu}\mathcal{P}_G^{0}(d)=\operatorname{Sym}(d)$, we define the Riemannian metric:
    \begin{equation*}
        g(X,Y) = \text{Tr}(X\Sigma Y).
    \end{equation*}
    It turns out that the distance function induced by this metric coincides with $W_2$ given in~\eqref{eq:w2g1}. We now write $\mathcal{P}_{G}^0(d)=\bigcup_{n\in\mathbb{N}}Y_n$ where
\begin{equation*}
    Y_n = \left\{ \ \mu_{\Sigma}\Bigm| \Sigma\in\operatorname{Sym}^+(d),\ \text{Tr}(\Sigma)\leq n \ \right\}.
\end{equation*}
Note that $Y_n$ is compact; this is because the sets of orthogonal matrices $O(d)$ and $\mathcal{D}=\{\text{diag}(\lambda_1,\ldots,\lambda_d)\mid  0\leq\lambda_i\leq n\}$ are both compact and the function $f_n:O(d)\times \mathcal{D}\to Y_n$ defined by $f_n(U,D)=UDU^{T}$ is continuous. Let $s=1/3$ and $X_n=\mathcal{P}_{G}^0(d)\setminus Y_{n}$. We will show that $W^2_2(x,y)>s$ for any $x\in X_{2n}$ and $y\in Y_{n}$ via the following lemma:
\begin{Lemma}
    For any $\Sigma_1\in\operatorname{Sym}^{++}(d)$ and $\Sigma_2\in \operatorname{Sym}^{+}(d)$, we have
    \begin{equation}\label{eq:trace}
        W_2(\mu_{\Sigma_1},\mu_{\Sigma_2}) \geq \left\lvert\text{Tr}(\Sigma_1)^{1/2}-\text{Tr}(\Sigma_2)^{1/2}\right\rvert.
    \end{equation}
\end{Lemma}
\begin{proof}
    Since $\Sigma_1$ is positive definite, we have that
    \begin{align*}
        \text{Tr}((\Sigma_1^{1/2}\Sigma_2\Sigma_1^{1/2})^{1/2}) &= \text{Tr}(\Sigma_1^{-1/2}(\Sigma_1^{1/2}\Sigma_2\Sigma_1^{1/2})^{1/2}\Sigma_1^{1/2}) \\
                                                                 &= \text{Tr}((\Sigma_2\Sigma_1)^{1/2}).
    \end{align*}
    Replacing $\Sigma_2$ by $t\Sigma_2$ for any $t>0$ yields
    \begin{equation*}
        W_{2}^2(\mu_{\Sigma_1},\mu_{t\Sigma_2})=\text{Tr}(\Sigma_1)+t^2\text{Tr}(\Sigma_2)-2t\text{Tr}((\Sigma_2\Sigma_1)^{1/2}) \geq 0.
    \end{equation*}
    Choosing $t=\text{Tr}((\Sigma_2\Sigma_1)^{1/2})/\text{Tr}(\Sigma_2)$ leads to $\text{Tr}((\Sigma_2\Sigma_1)^{1/2})\leq \text{Tr}(\Sigma_1)^{1/2}\text{Tr}(\Sigma_2)^{1/2}$. Therefore,
    \begin{equation*}
        W_{2}^2(\mu_{\Sigma_1},\mu_{\Sigma_2})\geq   \text{Tr}(\Sigma_1)+\text{Tr}(\Sigma_2)-2\text{Tr}(\Sigma_1)^{1/2}\text{Tr}(\Sigma_2)^{1/2}=\left(\text{Tr}(\Sigma_1)^{1/2}-\text{Tr}(\Sigma_2)^{1/2}\right)^{2}.
    \end{equation*}
\end{proof}

As a consequence, for any $x=\mu_{\Sigma}\in X_{2n}$ with $\Sigma\in\operatorname{Sym}^{++}(d)$ and $y\in Y_{n}$, we have
\[
    W_{2}(x,y)\geq \left\lvert\sqrt{2n}-\sqrt{n}\right\rvert > 1/3.
\]
Thus, if $x$ satisfies $B(x,s)\cap Y_n\not= \emptyset$, then we must have $x\in Y_{2n}$. This result can be extended to $x=\mu_{\Sigma}$ where $\Sigma\in\operatorname{Sym}^{+}(d)$: in view of~\ref{eq:trace}, we can make small perturbations on the eigenvalues of $\Sigma$ to obtain $\Sigma'\in\operatorname{Sym}^{++}(d)$ so that $W_2(\mu_{\Sigma},\mu_{\Sigma'})$ is arbitrarily small.

\bigskip
\begin{proof}[Proof of Theorem~\ref{thm:main3}]
Under the above observation, we are now ready to set up for the conditions in Lemma~\ref{lemma:helper}. Let $\ovl{g}$ be the standard Euclidean metric. For any $\Sigma\in \operatorname{Sym}^{+}(d)\subset\R^{d(d+1)/2}$, we denote by $\ovl{B}_{\ovl{g}}(\Sigma,r)$ the closed ball in the Euclidean space and $\ovl{B}_{g}(\mu_{\Sigma},r)$ the closed ball under the Riemannian distance $W_2$. Define a smooth map $\varphi:\mathcal{P}^0_G(d)\to\operatorname{Sym}^{+}(d)$ by $\varphi(\mu_{\Sigma})=\Sigma$. Our goal is to show the following: there exist $c,C>0$ such that, for any $x\in Y_{2n}$ and $y\in Y_{n}$,
\begin{equation}\label{eq:eqRE}
    c\lVert\varphi(x)-\varphi(y)\rVert_2\leq W_2(x,y)\leq C\lVert\varphi(x)-\varphi(y)\rVert_2.
\end{equation}
Since $\mathcal{W}_{2}$ is a \WPC space (Theorem~\ref{thm:w2wpc}), the universal consistency follows from Lemma~\ref{lemma:helper}.

Assume for a contradiction that the first inequality in~\eqref{eq:eqRE} is not true. Then we can find two sequences $(p_k)_{k\in\mathbb{N}}$ in $Y_{2n}$ and $(q_k)_{k\in\mathbb{N}}$ in $Y_{n}$ such that
\begin{equation}\label{eq:opp}
    \lVert\varphi(p_k)-\varphi(q_k)\rVert_2> kW_2(p_k,q_k),
\end{equation}
for all $k\in\mathbb{N}$. Thus, $\limsup_{k\to\infty} W_2(p_k,q_k)=0$. By passing to a subsequence, we may assume that $(p_k)_{k\in\mathbb{N}}$ and $(q_k)_{k\in\mathbb{N}}$ converges to the same point $p\in Y_{n}$. Since $\mathcal{P}_G^0(d)$ is locally compact, there is $r>0$ such that $\ovl{B}_g(p,r)$ is compact.

Let $r'=r/4$ and $\varepsilon\in (0,r')$. For any $x,y\in \ovl{B}_{g}(p,r')$, there exists a piecewise smooth curve $\gamma_{\varepsilon}:[0,1]\to\operatorname{Sym}^{+}(d)$ joining $x$ and $y$ such that $L_{g}(\gamma_{\varepsilon})<W_2(x,y)+\varepsilon$. Notice that $\gamma_{\varepsilon}$ lies entirely in $\ovl{B}_{g}(p,r)$: for any $t\in[0,1]$,
\[
    W_2(p,\gamma(t))\leq W_2(p,x)+W_2(x,\gamma(t))\leq W_2(p,x)+W_2(x,y)\leq 3r'<r.
\] 
For any $q\in \ovl{B}_{g}(p,r)$ and any $X\in T_{q}\mathcal{P}^0_{G}(d)=\operatorname{Sym}(d)$, we denote $\lVert X \rVert_g = \sqrt{g(X,X)}$. Since $\ovl{B}_{g}(p,r)$ is a compact set, there exists a constant $c,C>0$ such that $c\lVert X\rVert_{\ovl{g}}\leq\lVert X\rVert_{g}\leq C\lVert X\rVert_{\ovl{g}}$. Consequently,
\begin{equation*}
    L_{\ovl{g}}(\varphi(\gamma_{\varepsilon})) = \int_{0}^1\lVert \gamma'_{\varepsilon}(t)\rVert_{\ovl{g}} \ dt\leq c^{-1}\int_{0}^1\lVert\gamma'_{\varepsilon}(t)\rVert_{g} \ dt = c^{-1}L_{g}(\gamma_{\varepsilon}).
\end{equation*}
Taking the infimum over all such curves, we have $\lVert \varphi(x)-\varphi(y) \rVert_2<c^{-1}(W_2(x,y)+\varepsilon)$ for all $x,y\in\ovl{B}(p,r')$ and arbitrary $\varepsilon>0$. Thus, for a sufficiently large $k$, we have
\[
\lVert \varphi(p_k)-\varphi(q_k) \rVert_2<c^{-1}W_2(p_k,q_k),
\] 
which contradicts~\eqref{eq:eqRE}. Thus the first inequality in~\eqref{eq:eqRE} holds. The second inequality follows similarly by repeating the proof but switching $g$ and $\ovl{g}$.
\end{proof}

\subsection{Densities of finite wavelet series} \label{sec:wavelet}

Under the assumptions on wavelets given in Section~\ref{sec:mainthms}, we have the following inequalities from~\cite{Weed19a}.
\begin{Lemma}
    For measures $\mu_f$ and $\mu_g$ in $A_n$ where $f$ and $g$ are given in~\eqref{eq:wl},     
    \begin{align}
        W_1(\mu_f,\mu_g) &\leq C_1\biggl(\sum_{k=-K_0}^{K_0}|\alpha_{\ell k}-\alpha'_{\ell k}|+\sum_{j=\ell}^n\sum_{k=-K_j}^{K_j}2^{-\frac{3}{2}j}|\beta_{jk}-\beta'_{jk}|\biggr) \label{eq:upperwl} \\
        W_1(\mu_f,\mu_g) &\geq C_2\biggl(\sum_{k=-K_0}^{K_0}|\alpha_{\ell k}-\alpha'_{\ell k}|+\max_{\ell\leq j\leq n}\sum_{k=-K_j}^{K_j}2^{-\frac{3}{2}j}|\beta_{jk}-\beta'_{jk}|\biggr), \label{eq:lowerwl} 
    \end{align}
    for some positive constants $C_1$ and $C_2$.
\end{Lemma}

\begin{proof}[Proof of Theorem~\ref{thm:main2}]
    For a fixed $\mu_{f_0}\in \mathcal{V}_N$, define 
    \[\mathcal{G}_{f_0}(\mathcal{V}_N) = \{\mu_{(1-t)f_0+tg} \ | \ \mu_g\in \mathcal{V}_N, \quad t\in[0,1] \}.\]
    Thus, $\mathcal{G}_{f_0}(\mathcal{V}_N)$ contains constant speed geodesics under $W_1$ from $\mu_{f_0}$ to each measure in $\mathcal{V}_N$.
    
    Let $f\in \mathcal{G}_{f_0}(\mathcal{V}_N)$ with coefficients $\alpha_{\ell}=(\alpha_{-K_0},\ldots,\alpha_{K_0})$ and $\beta_j=(\beta_{j(-K_j)},\ldots,\beta_{jK_j})$, we define a function $\varphi:\mathcal{V}_N \to \R^{D_N}$ for a suitable $D_N$ as follows:  
    \[
        \varphi\left(\mu_{f}\right) = (\alpha_{\ell},\beta_{\ell},\ldots,\beta_{N}).
    \] 
    Given any $\mu_{g_1},\mu_{g_2}\in \mathcal{V}_N$, it follows from~\eqref{eq:upperwl} and~\eqref{eq:lowerwl} that
    \begin{align*}
        W_1(\mu_{g_1},\mu_{g_2}) &\leq C_1\biggl(2K_0+2\sum_{j=l}^N K_j\biggr)^{\frac{1}{2}}\|\varphi(\mu_{g_1})-\varphi(\mu_{g_{2}})\|_2
        \intertext{and}
        W_1(\mu_{f_0},\mu_{g_1}) &\geq C_2(N-\ell+1)^{-1}2^{-\frac{3}{2}\max_j K_j}\|\varphi(\mu_{f_0})-\varphi(\mu_{g_1}) \|_2.
    \end{align*}
    Thus, the $k_n$-NN classifier is universally consistent on $(\mathcal{V}_N,W_1)$ as a result of Lemma~\ref{lemma:helper}.
\end{proof}

\section{Conclusion and open problems}
We established that the $k_n$-NN classifier is not universally consistent on $\mathcal{W}_p((0,1))$ for any $p\geq 1$. Thus one cannot hope to obtain universal consistency without some restriction on the base metric space, or the Wasserstein space itself. We then give some examples of subsets of Wasserstein spaces, on which the $k_n$-NN is universally consistent. The first example is $W_p(X)$ for any $p\geq 1$, where $X$ is a $\sigma$-uniformly discrete set. The remaining examples exploit the geodesic structure of the Wasserstein spaces for $p=1$ and $p=2$. Specifically, we show that $k_n$-NN classifier is uniformly consistent on the space of measures supported on a finite set, the space of Gaussian measures, and the space of measures with densities expressed as finite wavelet series.

The following are related problems that might be worth exploring:
\begin{itemize}
    \item We have showed in Section~{\ref{sec:th1}} that, when $X$ is a finite set, the $k_n$-NN classifier is universally consistent on $\mathcal{W}_1(X)$. It is then natural to ask: does the universal consistency hold on $\mathcal{W}_p(X)$ for $p>1$?
    \item Does the universal consistency holds on other parametrized family of distributions, for example, the exponential family?
    \item We might instead consider the entropic regularized Wasserstein distance which can be computed much faster than the original Wasserstein distance~\cite{Cuturi2013}: 
        \begin{equation*}
            W_{p,\epsilon}(\mu,\nu) = \inf_{\pi\in\Pi^{\epsilon}(\mu,\nu)}\Big( \int_{X\times X} d(x,y)^p d\pi(x,y)\Big)^{1/p},
\end{equation*}
where $\Pi^{\epsilon}(\mu,\nu)$ is the set of probability measures on $X\times X$ with marginals $\mu$ and $\nu$ satisfying $D_{\text{KL}}(\pi^{\epsilon} \| \mu \otimes \nu)\leq \epsilon$. Can we obtain the same results presented in this paper if we replace $W_p$ by $W_{p,\epsilon}$ ?  
\end{itemize}

\section*{Acknowledgment}

The author would like to thank the reviewers for their comments which helped improve this work significantly. The author also would like to thank Chiang Mai University, Thailand, for financial support.


\bibliographystyle{imaiai}
%

\ifx\undefined\BySame
\newcommand{\BySame}{\leavevmode\rule[.5ex]{3em}{.5pt}\ }
\fi
\ifx\undefined\textsc
\newcommand{\textsc}[1]{{\sc #1}}
\newcommand{\emph}[1]{{\em #1\/}}
\let\tmpsmall\small
\renewcommand{\small}{\tmpsmall\sc}
\fi


\appendix

\section{Proof of Theorem~{\ref{thm:equiv}}}\label{sec:proofthm}
We start with a couple of definitions regarding measures on a metric space:
\begin{Def}
    For any metric space $(X,d)$, we denote by $\mathcal{M}(X)$ the set of all finite signed measures on $X$ and $\mathcal{M}^{+}(X)$ the set of all finite positive measures on $X$. Thus $\mathcal{M}^{+}(X)\subset\mathcal{M}(X)$ 

    For any $\rho\in\mathcal{M}^{+}(X)$ and $\eta\in\mathcal{M}(X)$, we define the quotient and the maximal function:
    \begin{align*}
        T_{r}(\eta,\rho)(y) &= \eta(\ovl{B}(y,r))/\rho(\ovl{B}(y,r)) \\
        S_{r}(\eta,\rho)(y) &= \sup_{0<a<r}T_{a}(\eta,\rho)(y).
    \end{align*}
    For any $f\in L^1(\rho)$, we denote by $f\rho$ the measure $A\mapsto \int_{A}f \ d\rho$.

    For any $\rho\in\mathcal{M}^{+}(X)$ and any set $A\subset X$ (not necessarily $\rho$-measurable), we define the upper measure $\ovl{\rho}(A)$ by 
    \[
        \ovl{\rho}(A) = \inf\left\{\sum_{i=1}^n \rho(E_i)\biggm|  E_i \text{ is measurable for all }i \text{ and }  A\subset \bigcup_{i=1}^n E_i\right\}.
    \] 
\end{Def}
The original proof of Assouad and Quentin de Gromard~\cite[Section 4]{Assouad2006} only assumes that $X$ is a set with a symmetric kernel $d$ (that is, $d$ does not have to satisfy the triangle inequality). For our applications, we make a stronger assumption that $(X,d)$ is a separable metric space and $\rho$ is a finite Borel measure, which allows us to simplify the proof of ($n_4$) $\Rightarrow$ ($n_5$) below.

First, we will prove that finite metric dimension implies the differentiation condition. The proof consists of the following statements for a metric space $(X,d)$:
\begin{itemize} 
    \item[($n_1$)] (Nagata dimension) For a given $Y\subset X$, there exists $s>0$ such that, for any $a\in X$ and $y_1,\ldots, y_{m+1}\in Y\cap B(a,s)$, there exists $i,j$ such that
        \begin{equation}
            d(y_i,y_j)\leq \max\{d(a,y_i),d(a,y_j)\}.
        \end{equation}
    \item[($n_2$)] (metric dimension) For a given $Y\subset X$, there exists $s>0$ such that, if $\mathcal{B}=\{\ovl{B}(y_i,r_i)\}_{i\in\mathbb{I}}$ is a family of closed balls with $r_i<s$ for all $i\in I$ and $y_i\in Y\setminus \ovl{B}(y_j,r_j)$ for all distinct $i,j\in I$, then $\mathcal{B}$ has multiplicity $\leq m$.
    \item[($n_3$)] (weak covering property) For a given $Y\subset X$, there exists $s>0$ such that, if $\mathcal{B}=\{\ovl{B}(y_i,r_i)\}_{i\in I}$ is a family of closed balls, where $\{y_i\}_{i\in I}\subset Y$ and $\{r_{i}\}_{i\in I}$ is contained in a decreasing sequence $(a_k)_{k\in\mathbb{N}}$ bounded above by $s$, then $\mathcal{B}$ has a subfamily of multiplicity $\leq m$ that covers $\{y_i\}_{i\in I}$.
    \item[($n_4$)] (maximal inequality) For a given $Y\subset X$, there exists $s>0$ such that, for any $\rho\in\mathcal{M}^{+}(X)$, any $\eta\in\mathcal{M}(X)$, any $r\in (0,s)$, and any $\alpha>0$, we have $\alpha\ovl{\rho}(Y\cap\{S_r(\eta,\rho)>\alpha\})\leq m\lvert\eta\rvert(X)$.
    \item[($n_5)$] (differentiation condition) Assume further that $(X,d)$ is separable. For any Borel $\rho\in\mathcal{M}^{+}(X)$ and any $f\in L^1(\rho)$, the quotient $T_r(f\rho,\rho)$ converges to $f$ $\rho$-almost surely on $Y$ as $r\to 0$.
\end{itemize}

Even though not necessary, ($n_1$) is provided here for completeness. We will prove that ($n_1$) $\Rightarrow$ ($n_2$) $\Rightarrow$ ($n_3$) $\Rightarrow$ ($n_1$) and ($n_3$) $\Rightarrow$ ($n_4$) $\Rightarrow$ ($n_5$). This proves Theorem~\ref{thm:equiv} for metric spaces with finite metric dimension, as any bounded $\rho$-measureable function, given that $\rho$ is finite, is in $L^1(\rho)$.

\bigskip
\begin{proof}[Proof of \emph{($n_1$) $\Rightarrow$ ($n_2$)}]
    Let $s$ be as in ($n_1$). Let $\{\ovl{B}(y_i , r_i )\}_{i=1}^k$ be a
    subfamily of $k$ closed balls centered in $Y$ containing a point $a\in X$. For any $i\not= j$, we have 
\[
    d(y_i,y_j) >\max\{r_i,r_j\}\geq \max\{d(a,y_i),d(a,y_j)\},
\] 
so ($n_1$) implies $k\leq m$.
\end{proof}

\bigskip
\begin{proof}[Proof of \emph{($n_2$) $\Rightarrow$ ($n_3$)}]
    Let $s$, $\mathcal{B}=\{\ovl{B}(y_i,r_i)\}_{i\in I}$ and $(a_k)_{k\in\mathbb{N}}$ be as in ($n_2$). Let $J_1$ be a maximal subset of $\{i\in I\mid r_i=a_1\}$ such that $d(y_i,y_j)>a_1$ for all distinct $i,j$ in $J_1$ (such $J_1$ exists because of the Hausdorff maximum principle). Suppose that $J_1,\ldots,J_{k-1}$ have been defined; we denote by $X_{k-1}$ the union of balls $\ovl{B}(y_j,r_j)$ over all $j$ in $\bigcup_{l=1}^{k-1}J_l$. We define $J_k$ to be a maximal set of $\{i\in I\mid r_i=a_k, y_i\notin X_{k-1}\}$ such that $d(y_i,y_j)>a_k$ for all distinct $i,j$ in $J_k$.

    Let $J$ be the union of all $J_k$'s. We observe that, for any ball $\ovl{B}(y_i,r_i)$ with $r_i=a_k$, if $y_i\notin X_{k-1}$ and $i\notin J_k$, then $y_i$ must be contained in $\bigcup_{j\in J_k}\ovl{B}(y_j,r_j)$ (otherwise we can add $i$ to $J_k$ which is maximal, a contradiction). Therefore, $\mathcal{B}_J=\{\ovl{B}(y_j,r_j)\}_{j\in J}$ is a subfamily of $\mathcal{B}$ containing $y_i$ for all $i\in I$. Moreover, by the construction, $y_{j}\notin \ovl{B}(y_{l},r_{l})$ for all distinct $j,l\in J$ and $r_j<s$ for all $j\in J$. Thus $B_J$ satisfies the conditions in ($n_2$). As a result, the multiplicity of $B_J$ is $\leq m$.
\end{proof}

\bigskip
\begin{proof}[Proof of \emph{($n_3$) $\Rightarrow$ ($n_1$)}]
    Let $s$ be as in ($n_3$). Let $a \in X$ and $y_1 , \ldots , y_k \in Y\cap B(a,s)$ with $d(y_i , y_j ) > \max\{d(a, y_i ) , d(a, y_j )\}$ for all distinct $i, j$. The balls $\ovl{B}_i=\ovl{B}(y_i , d(a, y_i ))$ satisfy $y_i \notin \ovl{B}_j$ for all distinct $i, j$. Thus, no proper subfamily of $\mathcal{B}=\{\ovl{B}_i\}_{i=1}^k$ contains all $y_1,\ldots, y_k$, which, combined with ($n_3$), implies that $\mathcal{B}$ itself must have multiplicity $\leq m$. Since $a\in\bigcap_{i=1}^k\ovl{B}_i$ is non-empty, we conclude that $k\leq m$.
\end{proof}

\bigskip
\begin{proof}[Proof of \emph{($n_3$) $\Rightarrow$ ($n_4$)}] 
    Let $s$ be as in ($n_3$). Let $r\in(0,s)$ and define
    \[Y^r_{\alpha}=\{y\in Y\mid S_r(\eta,\rho)(y)>\alpha\}.\]
    For any $y\in Y^r_{\alpha}$, there exists $r_y<s$ such that $\alpha\rho(\ovl{B}(y,r_y))<\eta(\ovl{B}(y,r_y))$. Using the continuity of measures, we assume that $r_y$ is rational. Write $\mathbb{Q}=\bigcup_{i\in\mathbb{N}}Q_i$, where $(Q_i)_{i\in\mathbb{N}}$ is an increasing sequence of finite sets and define 
    \[Y_{\alpha,j}=\{y\in Y^r_{\alpha}\mid r_y\in Q_j\}.\]
    Then $\mathcal{A}_j=\{\ovl{B}(y,r_y)\}_{y\in Y_{\alpha,j}}$ is a cover of $Y_{\alpha,j}$ whose radii are contained in a finite set $Q_j$ and are smaller than $s$. Thus it follows from ($n_3$) that $\mathcal{A}_j$ has a subcover $\mathcal{B}_j=\{\ovl{B}(y_i,r_{y_{i}})\}_{i\in I}$ with multiplicity $\leq m$. 

    We claim that $\mathcal{B}_j$ is countable: denoting $\mathcal{B}^n_j=\{A\in\mathcal{B}_j\mid \eta(A)>1/n\}$, we have
\begin{equation*}
    \frac{1}{n}\lvert\mathcal{B}^n_j\rvert < \sum_{A\in \mathcal{B}^n_j} \eta(A) \leq \sum_{A\in \mathcal{B}^n_j } \lvert\eta\rvert(A) \leq m\lvert\eta\rvert(X),
\end{equation*}
which implies $\lvert \mathcal{B}^n_j\rvert< nm\lvert\eta\rvert(X)$ for all $n\in\mathbb{N}$. Thus, as $\eta(\ovl{B}(y,r_{y}))>0$ for all $y\in Y^r_{\alpha}$, we can write $\mathcal{B}_j=\bigcup_{i\in\mathbb{N}}\mathcal{B}^n_{j}$ which is countable as claimed. Therefore, we have the following inequalities:
    \begin{align*}
        \alpha\ovl{\rho}(Y_{\alpha,j})&\leq \alpha\sum_{i\in I}\rho(\ovl{B}(y_i,r_{y_{i}}))<\sum_{i\in I}\eta(\ovl{B}(y_i,r_{y_{i}})) \\
                                     &\leq m\lvert\eta\rvert(\cup_{i\in I}(\ovl{B}(y_i,r_{y_{i}})) \leq m\lvert\eta\rvert(X).
    \end{align*}
    Taking the limit $j\to\infty$ gives $\alpha\ovl{\rho}(Y^r_{\alpha})\leq m\lvert\eta\rvert(X)$.
\end{proof}

\bigskip
\begin{proof}[Proof of \emph{($n_4$) $\Rightarrow$ ($n_5$)}] This is where our proof deviates from{~\cite{Assouad2006}}. Specifically, the original proof relies on a stronger version of Lemma{~\ref{lemma:contdense}} below, where $d$ is only assumed to be a symmetric kernel. In contrast, assuming that $(X,d)$ is a metric space allows us to obtain a constructive proof of Lemma{~\ref{lemma:contdense}}.
    \vspace*{-7mm}
    \begin{itemize} \item[]\begin{Lemma}\label{lemma:contdense} Let $(X,d)$ be a separable metric space and $\rho$ is a finite Borel measure on $X$. Then the set of bounded continuous functions is dense in $L^1(\rho)$.
\end{Lemma}
\begin{proof}
    Since $\rho$ is a finite Borel measure on a metric space, it is regular. Since $(X,d)$ is separable, the Borel $\sigma$-algebra is the same as the $\sigma$-algebra generated by closed balls in $X$. As the set of simple functions is dense in $L^1(\rho)$, it suffices to show that the function $\mathbf{1}_{C}$ for any closed set $C\subset X$ is an $L^1$-limit of a sequence of bounded continuous functions. We thus define 
    \[f_{C,n}(x)=\min\{1,n\cdot d(x,C) \},\] 
    which is continuous and bounded. By the dominated convergence theorem, $f_{C,n}\to \mathbf{1}_{C}$ in $L^1$ as $n\to\infty$.
\end{proof}
\end{itemize}

    For any $h\in L^1(\rho)$, we define $U_{\rho}h=\limsup_{r\to 0}\lvert T_r(h\rho,\rho)-h\rvert$. Let $f\in L^1(\rho)$. By Lemma~\ref{lemma:contdense}, for a given $\epsilon>0$, there exists a bounded continuous function $g$ such that $\lVert f-g\rVert_{L^1}<\epsilon$. By the continuity, we have $U_{\rho}g=0$. Let $h=f-g$. Then, with $s$ as in ($n_4$),
    \begin{equation*}
        U_{\rho}f \leq U_{\rho}g+U_{\rho}h\leq S_{s/2}(\lvert h \rvert \rho,\rho)+\lvert h \rvert,
    \end{equation*}
     Consequently, for any $\alpha>0$, we have $\{U_{\rho}>\alpha\}\subset\{S_{s/2}(\lvert h \rvert \rho,\rho)>\alpha/2\}\cup\{\lvert h \rvert>\alpha/2\}$. Combining this with ($n_4$) and the Chebychev's inequality yields:
\begin{align*}
    \alpha\ovl{\rho}(Y\cap\{U_{\rho}f>\alpha\}) &\leq \alpha\ovl{\rho}(Y\cap\{S_{s/2}(\lvert h \rvert \rho,\rho)>\alpha/2\})+\alpha\ovl{\rho}(Y\cap\{\lvert h \rvert>\alpha/2\})\\
                                              &\leq 2m\lVert h \rVert_{L^1}+2\lVert h \rVert_{L^1} \leq 2(m+1)\epsilon.
\end{align*}
Taking $\epsilon\to 0$ and then $\alpha\to 0$, we conclude that the set $\{y\in Y \mid U_{\rho}f(y)>0\}$ is a $\rho$-null set.
\end{proof}

\bigskip
We now extend the result to a metric space $(X,d)$ that has $\sigma$-finite metric dimension. Suppose that $X=\bigcup_{i\in\mathbb{N}}Y_i$ where each $Y_i$ satisfies either one of ($n_1$), ($n_2$) or ($n_3$). Since any of these statements implies ($n_5$), for any Borel $\rho\in\mathcal{M}^{+}(X)$ and any $f\in L^1(\rho)$, there exists a collection of $\rho$-null sets $\{E_i\}_{i\in\mathbb{N}}$ such that the quotient $T_r(f\rho,\rho)$ converges to $f$ on $Y_i\setminus E_i$ for all $i\in\mathbb{N}$. In other words, the convergence holds outside of the $\rho$-null set $\bigcup_{i\in\mathbb{N}}E_i$. \qed

\end{document}